\documentclass[11pt, letterpaper]{template}
\usepackage[authoryear, round]{natbib}
\usepackage{hyperref}[citecolor=magenta]

\hypersetup{
    colorlinks=true,
    citecolor=blue,
    linkcolor=blue,
    urlcolor=blue
}

\definecolor{darkblue}{rgb}{0, 0, 0.5}
\usepackage{url}            % simple URL typesetting
\usepackage{booktabs}       % professional-quality tables
\usepackage{amsfonts}       % blackboard math symbols
\usepackage{nicefrac}       % compact symbols for 1/2, etc.
\usepackage{microtype}      % microtypography
\usepackage{xcolor}         % colors

\usepackage{lineno}

\usepackage{enumitem}       % [leftmargin=20pt]
\usepackage{multirow}       % 
\usepackage{xspace}         % xspace
\usepackage{wrapfig}        % wrapfigure
\usepackage{makecell}       % divide row in table
\usepackage{colortbl}
\usepackage{bbm}            % indicator function (\mathbbm)
\usepackage{adjustbox}
\usepackage{amsthm}
\newtheorem{theorem}{Theorem}

\usepackage[most,skins,theorems]{tcolorbox}
\tcbset{
  aibox/.style={
    width=\linewidth,
    top=8pt,
    bottom=4pt,
    colback=blue!6!white,
    colframe=black,
    colbacktitle=black,
    enhanced,
    center,
    attach boxed title to top left={yshift=-0.1in,xshift=0.15in},
    boxed title style={boxrule=0pt,colframe=white,},
  }
}
\newtcolorbox{AIbox}[2][]{aibox,title=#2,#1}

\usepackage[utf8]{inputenc} % allow utf-8 input
\usepackage[T1]{fontenc} % use 8-bit T1 fonts
\usepackage{hyperref} % hyperlinks
\usepackage{url} % simple URL typesetting
\usepackage{booktabs} % professional-quality tables
\usepackage{amsfonts} % blackboard math symbols
\usepackage{nicefrac} % compact symbols for 1/2, etc.
\usepackage{microtype} % microtypography
\usepackage{xcolor} % colors
\usepackage{amsthm}
\usepackage{amsmath}
\usepackage{amssymb}
\usepackage{graphicx}
\usepackage{tcolorbox}
\usepackage{subcaption}
\usepackage{adjustbox}
\usepackage{enumitem}
\usepackage{wrapfig}
\usepackage{lipsum}
\usepackage[dvipsnames]{xcolor}
\usepackage{soul}
\usepackage{mdframed}
\usepackage{afterpage}
\usepackage{mathtools}

\newcommand{\cost}{\texttt{cost}}
\newtheorem{corollary}{Corollary}
\newtheorem{lemma}{Lemma}
\newcommand{\score}{\texttt{score}}
\newcommand{\E}{\mathbb{E}}

\newcommand{\Prob}{\mathbb{P}}
\newcommand{\Rcal}{\mathcal{R}} % Set of responses
\DeclareMathOperator*{\argmax}{\mathrm{argmax}}
\newcommand{\gst}{\ge_{\text{st}}} % Stochastic dominance

\usepackage{listings}
\title{Can Past Experience Accelerate LLM Reasoning?}

\author[ ]{Bo Pan\hspace{-2pt}}
\author[ ]{Liang Zhao\hspace{-2pt}}

\affil[ ]{\hspace{-2pt}Department of Computer Science, Emory University}

\affil[$\,$]{\hspace{-2pt}\texttt{\{bo.pan, liang.zhao\}@emory.edu}}  % note the non-empty label

\begin{abstract}
\textbf{Abstract  \ \ } Allocating more compute to large language models (LLMs) reasoning has generally been demonstrated to improve their effectiveness, but also results in increased inference time. In contrast, humans can perform tasks faster and better with increased experience and exposure. Hence, this paper aims to investigate the question: Can LLMs also become faster at reasoning through recurrent exposure on relevant tasks, and if so, how can it be achieved? To address these questions, we first formalize the problem setting of LLM reasoning speedup systematically in the dimensions of task relevancy and compute budget calculation. We then propose {SpeedupLLM}, a theoretically guaranteed framework to implement and benchmark such reasoning speedup behaviour based on adaptive compute allocation and memory mechanisms. We further conduct comprehensive experiments to benchmark such behaviour across different question similarity levels, memory methods, and reasoning methods. Results show that LLMs can generally reason faster with past experience, achieving up to a 56\% reduction in compute cost when equipped with appropriate memory and reasoning methods.
\end{abstract}

\begin{document}
\maketitle
\section{Introduction}

    Large Language Models (LLMs) have demonstrated reasoning capabilities to
    solve problems through step-by-step logical thinking
    \citep{brown2020language, wei2022chain}, which is crucial for applying LLMs to
    complex tasks in fields such as math and science \citep{achiam2023gpt}.
    Recently, research shows LLMs can better solve complex problems when allocated
    more compute at test time \citep{snell2025scaling}, and such techniques are
    referred to as \textit{test-time scaling} \citep{snell2025scaling, liu2025can}.
    However, increased compute also brings substantial computational overhead and increased reasoning time \citep{sui2025stop}, and this stimulates existing research on efficient reasoning algorithms
    \citep{sun2024fast, wang2025sampling, ding2025dynamic}, efficiency-oriented model
    fine-tuning \citep{luo2025o1, yu2distilling, kang2025c3ot, munkhbat2025self},
    model compression and distillation \citep{sun2025tinyr1, zhang2025reasoning}.

    For humans, repeated exposure on a specific task can lead to a significant reduction
    in the cognitive effort and time for execution \citep{shiffrin1977controlled,
    logan1988toward}, which is fundamental for humans to become proficient and efficient in
    various activities, from reading, motor skills, to complex problem-solving \citep{anderson1982acquisition}.
    However, it remains unknown whether LLM reasoning also has this merit. To
    fill this gap, this paper focuses on the question: \textbf{Can LLM reasoning
    be faster after past experience, and how can this be achieved?} In the remaining of this paper, we refer to LLMs' such potential behaviour as \textit{\textbf{reasoning speedup}}.

    There are two main reasons limiting LLM systems from achieving reasoning speedup:
    1) \textbf{Independent Question Processing}: Regular LLM systems simply
    process each query independently. However, nowadays, public LLM services accept millions of questions every hour, many of which are related or near-duplicate \citep{dammu2024near}, which are not fully leveraged to reduce redundant compute or accumulate useful experience. 
    % Without memory mechanisms, LLMs cannot improve
    % efficiency by leveraging past experiences. 
    2) \textbf{Static Compute Budget
    Allocation}: Existing test-time scaling methods do not adaptively allocate
    compute based on an LLM's proficiency to a question, thus hindering the model
    from becoming faster when meeting familiar questions. For example, Best-of-N
    sampling uses a fixed hyperparameter $N$, while tree-search methods, such as Tree-of-Thought, rely on a predefined maximum number of nodes to expand
    \citep{yao2023tree}.

    Therefore, to explore whether LLM can achieve reasoning speedup and how to achieve it, we propose {SpeedupLLM}, a unified framework to formulate, unify, and benchmark the behaviour of ``reasoning speedup over experiences'' across various LLM reasoning settings. Specifically, we first systematically formulate this question as a problem to explore the decreasing trend of reasoning cost across different question similarity levels and reasoning paradigms. 
    {SpeedupLLM} implements LLM reasoning speedup based on two key elements: 1) {Adaptive Compute Budget Allocation}, which extends various existing test-time scaling methods by early stopping with a threshold. 2) {Memory Mechanism}, which appends memory of previous questions and answers after processing each question. We conduct a theoretical analysis to prove that {SpeedupLLM} can enable reasoning speedup.

    We further conduct comprehensive experiments to benchmark different memory methods, test-time scaling methods on achieving LLM reasoning speedup at varying question similarity levels. 
    Experiments show that the LLM reasoning speedup behaviour generally exists across different memory and reasoning settings. For similar questions, the reasoning compute budget can be \textbf{reduced by up to 56\%} with the help of memory mechanisms.

    The contribution of this work includes:
    \begin{itemize}[leftmargin=.3in]
        \setlength{\itemsep}{0em}

        \item \textbf{New Problem}. We identify and systematically formulate the problem of reasoning speedup as the exploration of decreasing compute budget trends across varying question similarity levels and different reasoning methods.

        \item \textbf{Unified Framework}. We propose a unified and theoretically guaranteed framework, {SpeedupLLM}, to implement and benchmark LLM reasoning speedup based on memory mechanisms and adaptive compute allocation strategies, and it generally supports various reasoning methods.

        \item \textbf{Extensive Experiments}. We conduct benchmarking experiments across four test-time scaling methods, five memory methods, and four
            levels of question similarity.

        \item \textbf{Findings and Insights}. Our findings demonstrate that LLMs can achieve faster reasoning after experience. Such behaviour generally exists across different settings, and the reasoning compute budget can be {reduced by up to 56\%}. 
    \end{itemize}
\section{Related Work}
    \subsection{LLM Test-Time Scaling}
    \textbf{Test-Time Scaling} is the technique to improve LLMs' reasoning
    ability on complex questions by allocating more compute at the test time
    \citep{snell2025scaling}, and it has received increasing attention from the research
    community \citep{parashar2025inference, wu2025inference, ji2025test, li2025system,
    liu2025can}. Current representative test-time scaling methods include 1) parallel
    scaling methods, e.g., Best-of-N sampling \citep{stiennon2020learning}, which samples multiple complete answers and selects the highest-scoring one as the
    final output, and Self Consistency \citep{wang2022self}, which generates multiple answers and select the most common one; 2) Sequential Scaling
    methods, e.g. Self-Refine \citep{madaan2023self, gou2023critic}, which gradually
    refine the answer based on internal or external feedback; 3) Tree Search methods
    \citep{yao2023tree, feng2023alphazero, guan2025rstar}, which usually form
    each reasoning step as a node, and conduct tree search algorithms to search for
    optimal reasoning chains; 4) Long Chain-of-Thought methods, e.g., the reasoning of OpenAI GPT-4o \citep{jaech2024openai} and DeepSeek-R1 \citep{guo2025deepseek}, which conduct implicit searching in the text space by generating long reasoning chains.

    \subsection{LLM Memory}
    \textbf{Memory mechanisms} enable LLMs to retain and use information to generate responses \citep{zhang2024survey}, and such information can be from past experience or an external knowledge base \citep{zeng2024structural}. There are two main forms of memory: parametric form and textual form \citep{zhang2024survey}.
Parametric-form memory stores the memory in model weights, with representative methods including supervised fine-tuning (SFT) \citep{hu2022lora, shao2023character}, which finetunes the LLMs with past inputs and outputs; and knowledge editing \citep{de2021editing, mitchell2021fast, fang2024alphaedit}, which mainly focuses on injecting factual knowledge. Textual-form memory saves textual information as memory; the content can be original past interactions \citep{li2023long, huang2023memory, liu2023lost, zhong2024memorybank, zheng2023synapse}, reflection (insights extracted from past interactions) \citep{shinn2023reflexion, renze2024self, yang2023failures, hui2024rot}, and atomic facts \citep{anokhin2024arigraph, li2024graphreader}. In this work, we focus on leveraging past experience to enhance efficiency, so we consider past experience-based memory, including SFT, textual-form past experiences, and reflection on past experiences. We exclude fact-oriented memory structures (e.g., atomic fact databases), as they are primarily designed for knowledge recall rather than past experience.

    % \subsection{LLM Self-Evolution}
    % \textbf{Self-evolution} refers to the approaches that enable LLMs to autonomously
    % acquire, refine, and learn from experiences generated by the model itself \citep{tao2024survey}.
    % The basic idea of LLM self-evolution is to update the model with its past
    % experience, with representative methods including \citep{yu2distilling, singh2025self,
    % wang2025self, guan2025rstar}. This work can also be classified to the study
    % of LLM self-evolution. Prior work has primarily focused on improving answer quality,
    % and our study falls under the broader scope of LLM self-evolution but
    % targets a fundamentally different goal: the evolution of inference efficiency.

    \subsection{LLM Efficient Reasoning}

    Although test-time scaling significantly boosts LLMs' reasoning ability, it also results in substantial computational overhead and increased reasoning time \citep{sui2025stop}. Therefore, various types of LLM efficient reasoning methods have been developed. Existing methods can be categorized into RL-optimization with length reward \citep{luo2025o1, aggarwal2025l1, team2025kimi}, SFT with shorter CoT \citep{yu2distilling, kang2025c3ot, liu2025can}, latent representation compression \citep{hao2024training, cheng2024compressed, xu2025softcot}, dynamic reasoning algorithms \citep{sun2024fast, wang2025sampling, ding2025dynamic}, prompt-guided efficient reasoning \citep{han2024token, lee2025well, xu2025chain}, training data efficiency methods \citep{ye2025limo, muennighoff2025s1}, model compression and distillation \citep{sun2025tinyr1, zhang2025reasoning}. To our best knowledge, no existing research has focused on exploring the efficiency brought by memory or exposure to similar questions.
\section{Methodology}
 \subsection{Problem Formulation}
    \label{sec:problem} In this work, our central question is whether it is possible that, an LLM can gradually become faster when answering multiple similar questions. 
    Formally, let $f$ be an LLM, and let
    $\mathcal{Q}= \bigl(q^{(1)}, q^{(2)}, \dots, q^{(N)}\bigr)$ be a set of $N$
    test questions, where each $q^{(n)}$ is a natural language question. We study how the similarity among questions in $\mathcal{Q}$ affects the model’s reasoning efficiency. 
    
    To this end, we define
    four levels of \textit{similarity} of a group of questions from most similar
    to least similar, as shown in Table~\ref{tab:sim_def}.
    \begin{table}[h]
        \small
        \centering
        \caption{Definition of question similarity levels, from most similar (\textbf{S1})
        to least similar (\textbf{S4}).}
        \begin{tabular}{cp{6cm}p{8cm}}
            \toprule \textbf{Level} & \textbf{Description}                                       & \textbf{Example}                                                                             \\
            \midrule S1             & Exactly the same questions.                                & ``What is 3+4?'' vs ``What is 3+4?''                                                         \\
            S2                      & Same numbers, different wording                            & ``What is 3+4?'' vs ``What is the sum of 3 and 4?''                                          \\
            S3                      & Same structure, different numbers                          & ``What is 3+4?'' vs ``What is 5+6?''                                                         \\
            S4                      & Same underlying knowledge, different structure and numbers & ``What is 3+4?'' vs ``If John has 2 apples and Mary has 3, how many do they have together?'' \\
            \bottomrule
        \end{tabular}
        \label{tab:sim_def}
    \end{table}

    To measure the reasoning efficiency, we define \textit{compute budget}, $\texttt
    {cost}(f(q^{(n)}))$ for answering $q^{(n)}$, as the number of
    conducted operations in each test-time scaling method's dominant scaling dimension, e.g., the
    number of sampled answers in Best-of-N, and the number of nodes expanded in tree
    search-based methods.

    Our primary objective is to investigate the trend of compute budgets $\{\texttt
    {cost}(f(q^{(1)})),$ $\texttt{cost}(f(q^{(2)})),$ $...,$ $\texttt
    {cost}(f(q^{(N)}))\}$ given the question set $\mathcal{Q}$, and identify conditions under which there can be a decreasing trend, considering varying levels of question similarity, different test-time scaling strategies, and memory mechanisms.

    \subsection{SpeedupLLM: A Unified Framework for Implementing and Benchmarking LLM Reasoning Speedup}

    To implement and benchmark LLM reasoning speedup, we propose {SpeedupLLM}, a theoretically guaranteed framework that can give a decreasing trend on reasoning cost for relevant questions, based on adaptive compute budget allocation and memory mechanism.

    \subsubsection{Framework Design}
    \textbf{Preliminary.} 
    We first give a general form of existing test-time scaling methods. Let $\mathcal{S}$ be the set of \emph{test-time scaling methods}. When the model process the question $q^{(t)}\in\mathcal{Q}$ using a test-time scaling method $s\in \mathcal{S}$, it generates multiple candidate answers in the reasoning process as $\mathcal{R}_{s}^{(t)}=f_{s}(q^{(t)})$, where $f_s$ means generating using test-time scaling method $s$ with model $f$,  $\mathcal{R}_{s}^{(t)}= \{r_{1; s}^{(t)}, r_{2; s}^{(t)}, \dots, r_{n_t; s}^{(t)}\}$, and each $r_{k;s}^{(t)}\in\mathcal{R}_{s}^{(t)}$ is a candidate answer. An evaluation function $\texttt{score}(\cdot)$ estimates the quality of each candidate, and the final answer is selected via 
        $\argmax_{r\in\mathcal{R}_{s}^{(t)}}\; \texttt{score}(r).$
    
    In Appendix~\ref{append:scaling_method}, we show how various existing test-time scaling methods can be unified into this form.

    \textbf{Adaptive Compute Budget Allocation.} To extend existing test-time scaling methods to adaptively allocate compute budget based on the model's proficiency, we aim at strategies to early-stop the generation once a satisfying answer has been generated. This allows a smaller candidate set to achieve the same maximum score as the full set. Formally, we use $s'$ to denote the adaptive extension of method $s$. Then $f_{s'}$ is attained by:
\begin{equation}
    \min_{f_{s'}(q^{(t)})} 
    \texttt{cost}(\mathcal{R}_{s'}^{(t)})
    \quad \text{s.t.} \;
    \max_{r \in \mathcal{R}_{s'}^{(t)}}(\texttt{score}(r)) \ge \tau,\ \text{where }\mathcal{R}_{s'}^{(t)} \coloneq f_{s'}(q^{(t)}) \subseteq \mathcal{R}_{s}^{(t)}
    \label{eq:mobj}
\end{equation}
    where $\tau$ is the threshold that a score is considered as satisfying, and $\mathcal R_{s'}^{(t)} =f_{s'}(q^{(t)})$ is the reduced generated candidate set. In Section~\ref{sec:specification}, we show how various existing test-time scaling methods can be extended to adaptively allocate compute budget under this formulation.

    % \textbf{Reasoning Speedup via Memory.} 
    \textbf{Reasoning with Memory Mechanism.} Next, we incorporate the memory mechanism into the reasoning process.
    Let $\mathcal{M}$ be the set of \emph{memory methods}. 
    During inference, the model processes each question $q^{(t)} \in \mathcal{Q}$ sequentially, leveraging the current memory state $\mathbf{M}^{(t)}$ as
    \begin{equation}
        \mathcal{R}_{s';m}^{(t)}\;=\; f_{s'}\bigl(q^{(t)};\mathbf{M}
        ^{(t)}\bigr),
    \end{equation}
    where $m\in \mathcal{M}$ is the employed memory method. The maintained memory $\mathbf{M}^{(t)}$ is constructed from prior question-answer pairs by
    \begin{equation}
        \label{eq:memory}\mathbf{M}^{(t)}= g_{m}\bigl(\{(q^{(i)}, \mathcal{R}_{s';m}^{(i)}
        )\}_{i=1}^{t-1}\bigr)\quad\forall t>1.
    \end{equation}
    where $g_m$ is the function to save memory with the memory method $m$.

    \subsubsection{Theoretical Analysis}
    Here, we prove that {SpeedupLLM} can reduce reasoning cost as the model experiences more relevant questions.
     We begin by showing that, under adaptive
    compute budget allocation, if the answer quality improves over time, then the compute budget decreases accordingly, as shown in Theorem~\ref{thm:1}.

    \begin{theorem}[Non-Increasing Compute Budget with Non-Decreasing Answer Quality]
        \label{thm:1} 
        If under the test-time scaling method $s$, the probability that the best response exceeds the quality threshold, i.e. $\Prob\bigl(\max_{r\in \Rcal_{s} ^{(t)}}{\normalfont\texttt{score}}(r) \ge \tau\bigr)$, is non-decreasing with $t$,
        then the expected compute budget $\E[{\normalfont\cost} (\mathcal{R}_{s'}^{(t)})]$ when
        using $s'$ should be non-increasing with $t$.
    \end{theorem}
\begin{proof}
    Detailed proof is elaborated in Appendix~\ref{proof:1}.
\end{proof}
    
        Next, we show that the memory mechanism can help improve answer quality. 
        
\begin{theorem}[Non-Decreasing Answer Quality with Accumulating Relevant Memory]\label{thm:2}
    If additional correct memory from relevant questions does not degrade model performance, i.e., if 
    \[
        \mathbb{E}_{r \in f_{s_{j}}(q^{(t)};{\cal M})}\bigl[{\normalfont\texttt{score}}(r)
        \bigr] \ge \mathbb{E}_{r' \in f_{s_{j}}(q^{(t)};{\cal M}')}\bigl
        [{\normalfont\texttt{score}}(r') \bigr],
    \]
    for ${\cal M}'\subseteq{\cal M}$, then for any $k$, the probability of a satisfying answer appears in the first $k$ candidates, i.e., $\Prob\left(\max_{1 \le i \le k}{\normalfont\texttt{score}}(r_{i; s}^{(t)})\ge\tau\right)$, is non-decreasing by $t$. 
\end{theorem}
\begin{proof}
    Detailed proof is elaborated in Appendix~\ref{proof:2}.
\end{proof}
% \textit{Proof of Theorem~\ref{thm:2}.} 
    Combining Theorem~\ref{thm:1} and Theorem~\ref{thm:2},  we have
    \begin{corollary}\label{corollary}
    Under the conditions in Theorem~\ref{thm:1} and \ref{thm:2}, {SpeedupLLM} achieves a
non-increasing expected compute budget while maintaining the
probability of producing a satisfying answer.
    \end{corollary}
    
    It shows that, by integrating adaptive compute allocation and memory-augmented reasoning, {SpeedupLLM} can yield a decreasing trend in the compute budget of reasoning.

    \subsection{Specification to Different Test-Time Scaling Methods}\label{sec:specification}

    % \textcolor{red}{check the google paper and see how they introduce methods}

    Next, we show how {SpeedupLLM} can be specified for different test-time scaling strategies. We explore four representative streams of test-time scaling methods: Best-of-N, Tree-of-Thoughts \citep{yao2023tree}, Self-Refine \citep{madaan2023self}, and Long Chain-of-Thought (Long CoT) \citep{chen2025towards, guo2025deepseek}. These modifications make each scaling strategy adaptive, allowing them to allocate compute budget based on the LLM's familiarity with the questions.

    % \textcolor{red}{introduce what is R}

    \textbf{Best-of-N}. In this method, the minimal unit of generation after which a complete answer can be evaluated is a whole answer. Thus, $\texttt{cost}(\cdot)$ is measured by the number of generated and evaluated answers. $\texttt{score}(\cdot)$ is provided by an LLM Judge \citep{zheng2023judging} or a Process Reward Model \citep{lightman2023let, zhang2025lessons}. Optimizing Eq.~\ref{eq:mobj} is practically performed by sequentially (or in batches) generating and scoring each answer. The generation stops once an acceptable answer is produced.

\textbf{Tree-of-Thoughts}. In this method, the minimal unit of generation after which a complete answer can be evaluated is a node in the search tree (often representing a reasoning step). Thus, $\texttt{cost}(\cdot)$ is evaluated as the number of generated and evaluated nodes. Similar to Best-of-N, $\texttt{score}(\cdot)$ is given by an LLM Judge \citep{zheng2023judging} or a Process Reward Model \citep{lightman2023let, zhang2025lessons}. To practically optimize Eq.~\ref{eq:mobj}, when expanding each node in the tree search, we sequentially evaluate each node; once we encounter a node with an above-threshold score, we prune the following nodes and expand the current node. Note that this approach also unifies DFS and BFS.

    \textbf{Self-Refine}. This method is intrinsically compute adaptive.
    In this method, the minimal unit is one whole refined answer. Thus, $\texttt{cost}(\cdot)$ is evaluated as the number of generated and evaluated answers.
    $\texttt{score}(\cdot)$ is given by the model itself or an external LLM. Since this
    method automatically stops when a satisfying answer appears, there is no
    specific modification required to optimize Eq.~\ref{eq:mobj}.

    \textbf{Long CoT}. Long CoT reasoning conducts free-form reasoning processes that implicitly incorporate self-refinement and tree-search strategies within the text generation space
    \citep{chen2025towards}. This approach allows the model to continue generating content until it determines that a satisfactory answer has been reached, eliminating the need for predefined stopping tokens such as "wait" \citep{sui2025stop}; thus, it is also an inherently adaptive method. Since this is a text-space reasoning method, the $\texttt{cost}(\cdot)$ is evaluated as the total number of generated
    tokens, and $\texttt{score}(\cdot)$ is assessed by the model's own estimation of the next-token probability, i.e., to generate another ``wait'' to continue thinking or stop with the current answer. This method also does not require specific modification to be compute budget-adaptive.

 \section{Experimental Setup}

    \textbf{Benchmarking Dimensions.} To explore the LLM reasoning speedup behaviour, we conduct experiments along three
    dimensions: 1) Question Similarity, 2) Memory Method, and 3) Scaling Method.
    The details are introduced as follows.

    \textbf{Data.} In this study, we create a dataset covering four levels of similarity
    as defined in Section~\ref{sec:problem}.
    We sample 10 questions from the MATH dataset to serve as the backbone to create
    similar questions. For each of these four similarity levels, these backbone questions
    are extended into a set of 20 questions per backbone through a combination of large
    language models (LLM) and programming-based calculations of the results, to
    ensure the correctness of the answer. Examples of questions
    with different similarities are given in Appendix~\ref{append:data}.

    \textbf{Memory Methods.} In this study, we explore both parametric-form and textual-form memory \citep{zhang2024survey}.
    In addition to the baseline of no memory mechanism, we conduct experiments to
    evaluate five memory methods: one parametric method (SFT) and four text-based
    memory methods. 

    \begin{itemize}[leftmargin=.3in]
        \setlength{\itemsep}{0em}

        \item \textbf{No Memory} (Baseline): Questions are processed individually without memory mechanism.
        \item \textbf{SFT} (Supervised Fine-Tuning) as memory \citep{shao2023character, wang2023huatuo, yang2023investlm}: The past question
            and generated answer pairs are used as data to perform supervised fine-tuning
            on the model.

        \item \textbf{In-Context} \citep{zhao2024expel, huang2023recommender, wang2023enhancing}: The past questions and generated answers are used as in-context examples to guide the reasoning.

        \item \textbf{Reflection}: LLMs self-summarize past experiences and summarize
            rules to guide the reasoning. We test three variants: 1) \textbf{Reflect} (individually) \citep{zhong2024memorybank, packer2023memgpt}: The LLM individually reflects on each previous question
            and answer pair and summarizes experience from each one. 2) \textbf{Multi-Case
            Reflect} \citep{zhao2024expel, tack2024online}: The LLM is asked to summarize experience from multiple
            previous question and answer pairs. 3) \textbf{Reflect-Update} \citep{hu2023chatdb, shinn2023reflexion}: The LLM
            maintains a reflection by updating it with each new question and answer
            pair.
    \end{itemize}
    \textbf{Metrics.} We evaluate 1) allocated \textit{compute budget} and 2) \textit{accuracy} on
    each question. For each single question, these two metrics are averaged over four runs.

    \textbf{Implementation Details.} See Appendix~\ref{append:implement}.

    \section{Results}

    \begin{figure*}[t]

        \centering
        \begin{tabular}{c}
            \includegraphics[width=1.0\textwidth, trim=218 36 80 20, clip]{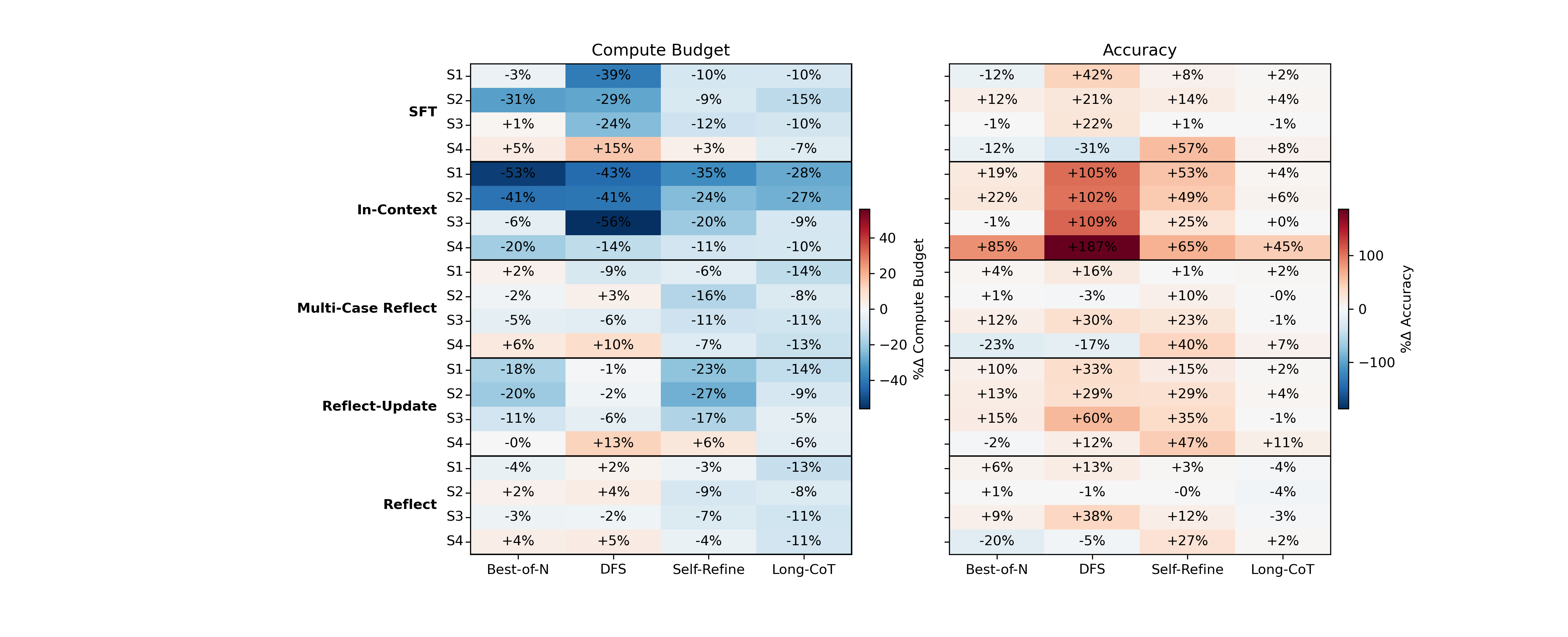}
        \end{tabular}

        \caption{Changes in compute budget and accuracy relative to the baseline
        without memory mechanisms, across each memory method, test-time scaling
        method, and question similarity level. Values represent the relative
        compute budget and accuracy, averaged over all question backbones and variations.}
        \label{fig:avg_heatmap}
        % \vspace{-3mm}
    \end{figure*}

    \afterpage{\clearpage}
    \begin{figure}[p]
        \centering
        % Left subplot (cropped)
        \begin{subfigure}
            [b]{0.45\textwidth}
            \includegraphics[width=\linewidth, trim=60 0 20 0, clip]{
                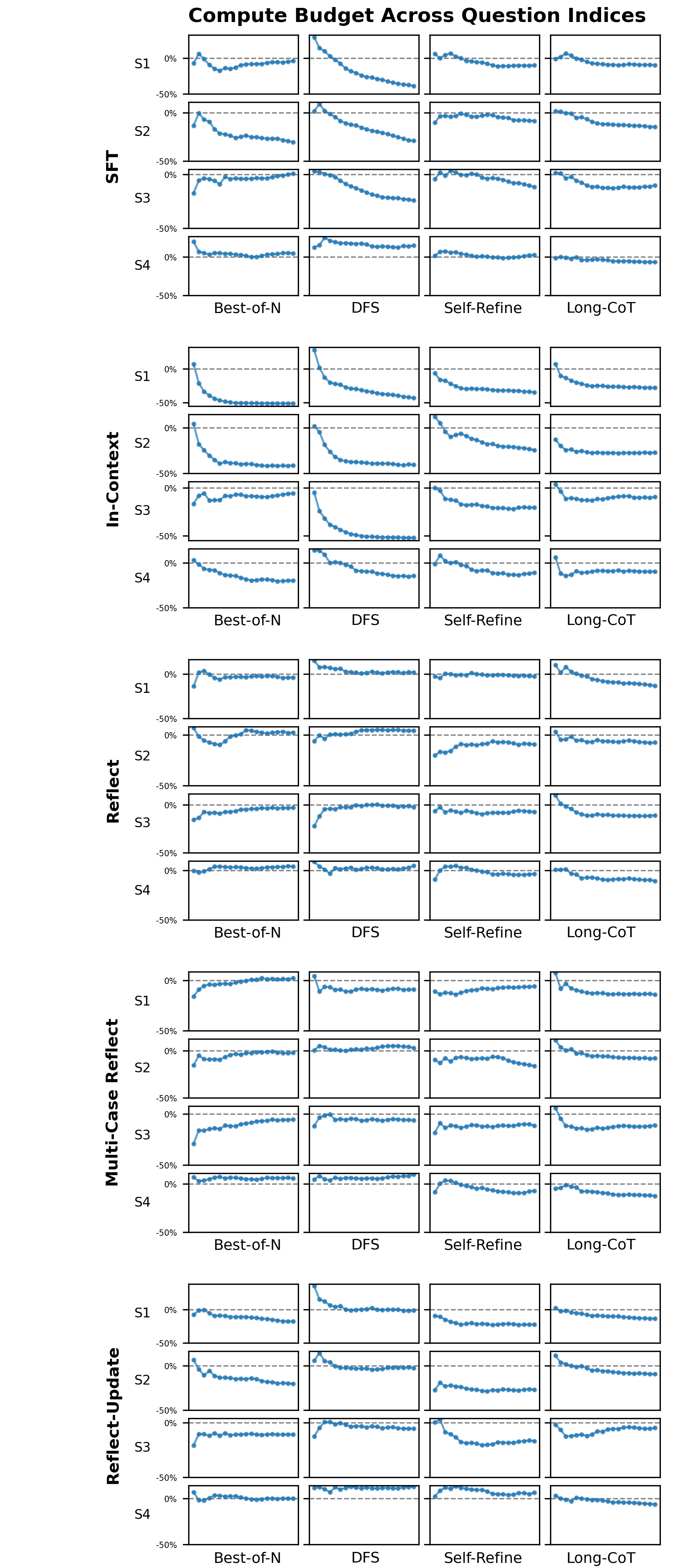
            }
            \caption{\textbf{Allocated Compute}}
        \end{subfigure}
        \hfill
        % Right subplot (cropped)
        \begin{subfigure}
            [b]{0.45\textwidth}
            \includegraphics[width=\linewidth, trim=60 0 20 0, clip]{
                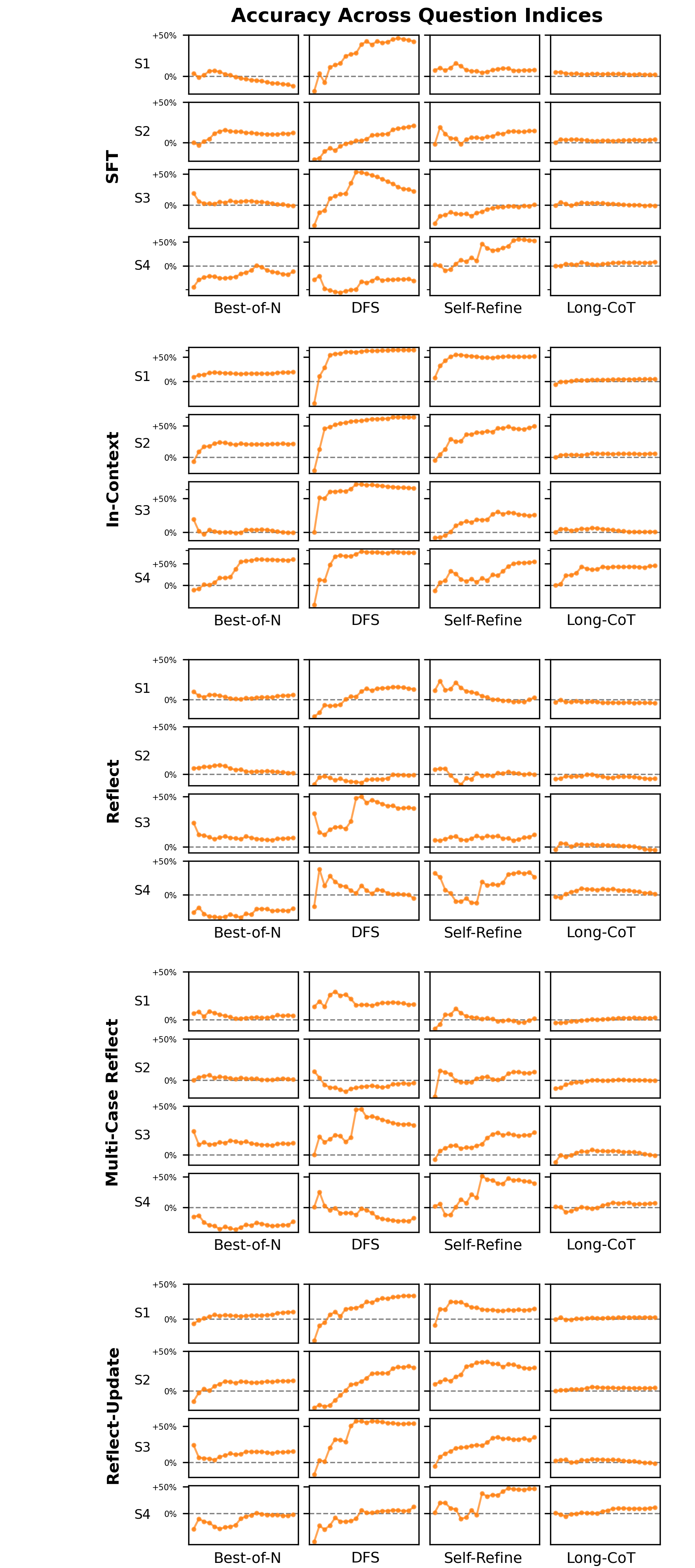
            }
            \caption{\textbf{Accuracy}}
        \end{subfigure}

        \caption{Changes in (a) compute budget and (b) accuracy relative to the
        baseline (no memory), across different memory methods, test-time scaling
        methods, and question similarity levels. In each subplot, the x-axis denotes
        the question index within a sequence of questions, while the y-axis shows the
        percentage change in compute budget or accuracy compared to the baseline.
        Each curve presents values averaged over multiple question sets at each question
        index. Gray dashed lines represent baseline (no memory) performance.}
        \label{fig:all_curve}
    \end{figure}

    % \textcolor{red}{Each one need a technical explanation}
    % \textcolor{red}{Each one need to tell people where to look at}

    % We first present our main results in Fig.~\ref{fig:avg_heatmap} and Fig.~\ref{fig:all_curve}.
    % Fig.~\ref{fig:avg_heatmap} shows the relative compute budget and accuracy
    % with the percentage change compared to the baseline without memory
    % mechanisms. Results are averaged over 10 question sets, each containing 20
    % questions. Each group of four rows (S1 to S4) corresponds to a memory method,
    % while each column represents a test-time scaling method. Color intensity indicates
    % the degree of improvement or degradation relative to the baseline.

    % Fig.~\ref{fig:all_curve} presents more detailed dynamics of compute budget
    % and accuracy changes of each method. It illustrates how different memory methods
    % and test-time scaling strategies affect inference efficiency and accuracy over
    % the course of sequential question answering, under varying levels of
    % question similarity. The left panel (a) shows the percentage change in compute
    % budget, while the right panel (b) shows the change in accuracy, both relative
    % to the baseline without any memory mechanism. Within each subplot, the x-axis
    % denotes the query index within the sequence, and the y-axis reports the
    % relative change compared to the baseline. Results are averaged across 10 question
    % sets, each containing 20 questions, ensuring robustness of the trends. Gray
    % dashed lines of 0\% indicate the baseline (no memory).

    We first present the compute budget trends across different levels of question similarity, reasoning strategies, and memory methods Fig.~\ref{fig:all_curve}. For a clearer comparison, we aggregate the compute budget over question indices within each question set, as shown in Fig.~\ref{fig:avg_heatmap}. The following sections provide detailed analyses along each experimental dimension. We first answer the central question of the possibility of LLM reasoning speedup in Section~\ref{res:possibility}, then we explore the question similarity's effect in Section~\ref{res:similarity}, the memory method's effects in Section~\ref{res:memory}, and finally the test-time scaling method's effects in Section~\ref{res:scaling}.

    \subsection{Possibility of LLM Reasoning Speedup}\label{res:possibility}

    % \noindent\underline{\textbf{Finding 0:} Similarly to humans, LLMs can achieve automaticity through practice.}

    \noindent
    \textit{\textbf{Finding 1: LLMs can achieve
    reasoning speedup through past experience.}}
    
    As shown in Fig.~\ref{fig:avg_heatmap}, the left panel demonstrates significant
    reductions in compute budget, with up to a 56\% reduction when using the combination of \textit{In-Context} memory and \textit{DFS} reasoning in the similarity level S3, and there is at least a 10\% reduction in 47.5\% settings. Additionally, such reasoning speedup behaviour consistently occurs across 64 out of 80 settings (\sethlcolor{Cyan!15}\hl{cool cells} in left panel), demonstrating that reasoning speedup is a general behaviour for LLMs when equipped with memory and adaptive compute allocation.

    \vspace{1ex}
    \noindent
    \textbf{\textit{Finding 2. Response speed and accuracy are strongly
    correlated; faster responses tend to be more accurate.}}

    Fig.~\ref{fig:avg_heatmap} shows a clear correlation between reasoning
    efficiency (compute budget) and accuracy. Specifically, regions with deeper blue in the left
    panel (more reductions in compute cost) often correspond to regions with deeper
    red in the right panel (more accuracy improvements). Upon examination, the relative
    compute budget and accuracy have a Pearson correlation of -0.41 with p=0.0002,
    suggesting a moderate and statistically significant negative correlation. These observations suggest
    that enhancing reasoning speed does not sacrifice, but instead improves answer correctness. This is because test-time scaling methods suffer from the gap between the estimated answer quality score and actual correctness, but with correct memory, this issue can be alleviated, leading to improved accuracy.

    \subsection{Question Similarity: The Boundary of Reasoning Speedup}\label{res:similarity}

    \begin{wrapfigure}
        {r}{0.45\textwidth}
        \centering

        \includegraphics[width=\linewidth]{
            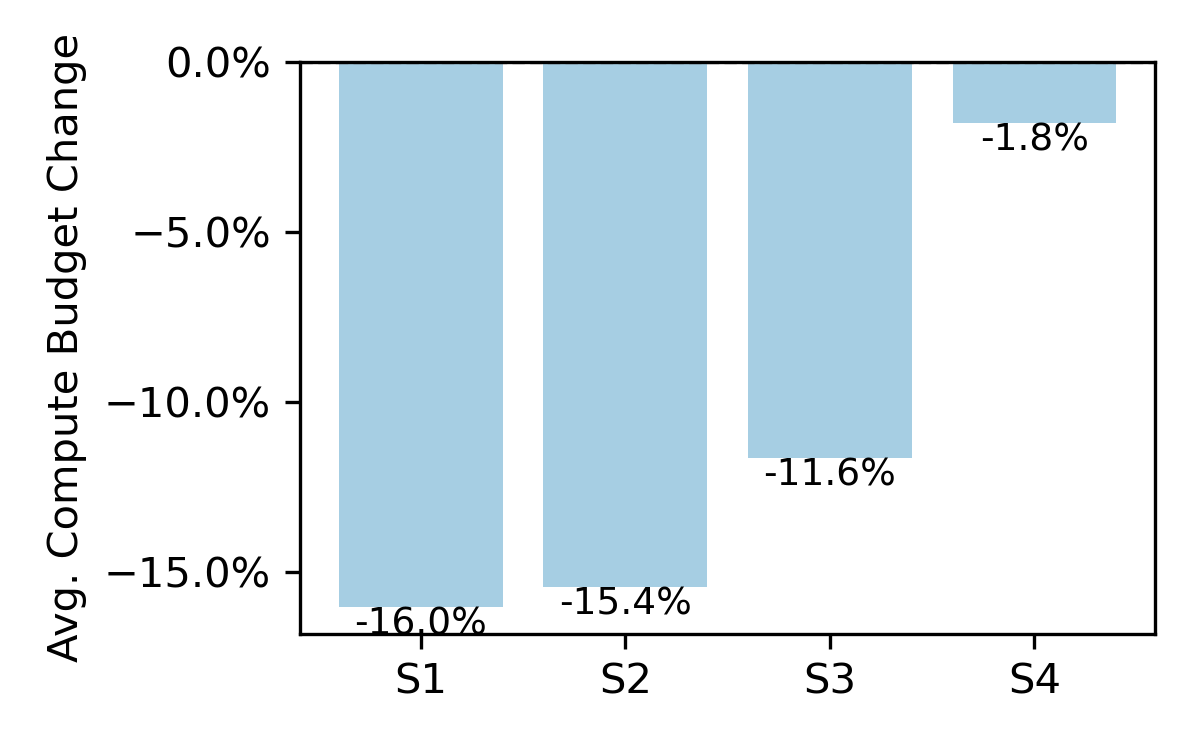
        }

        \caption{Compute budget changes grouped and averaged by question
        similarity levels.}

        \label{fig:sim_avg}
    \end{wrapfigure}

    \textit{\textbf{Finding 3. Reasoning efficiency gains increase with question
    similarity.}}

    We also find that reasoning speedup is more pronounced when questions are more similar. To highlight this effect, we present Fig.~\ref{fig:sim_avg},
    where all results from Fig.~\ref{fig:avg_heatmap} are grouped and averaged by
    their similarity levels. As shown, the reductions in compute budget are most significant in more similar question groups (S1: 16.0\%, and S2: 15.4\%). As more
    details are shown in Fig.~\ref{fig:avg_heatmap}~(left), such pattern is consistent
    across most memory methods, especially \textit{SFT}, \textit{In-Context}, and \textit{Reflect-Update}.
    This pattern is largely expected, since it aligns with patterns observed in human cognition, where efficiency improves larger for more similar tasks.

    \noindent
    \textit{\textbf{Finding 4. When question similarity is low, memory mechanisms can cause a performance drop.}}

    While memory mechanisms generally reduce compute costs for similar questions, Fig.~\ref{fig:avg_heatmap} shows that under low similarity, some methods can increase compute costs and degrade accuracy, as indicated by \sethlcolor{Peach!20}\hl{warm cells} in the left panel and \sethlcolor{Cyan!15}\hl{cool cells} in the right panel, such patterns mostly appears in low similarity question groups (S4). This is because when the
    questions and answers in memory differ substantially from the current query,
    the model can overfit to irrelevant examples from memory
    \citep{zhang2024uncovering}, and repeated reliance on a narrow set of memories can trigger catastrophic forgetting, reducing the model’s ability to generalize
    \citep{luo2023empirical}.

    \subsection{The Effects of Memory Methods}\label{res:memory}

    \begin{wrapfigure}
        {r}{0.45\textwidth}
        \centering
        \includegraphics[width=\linewidth]{
            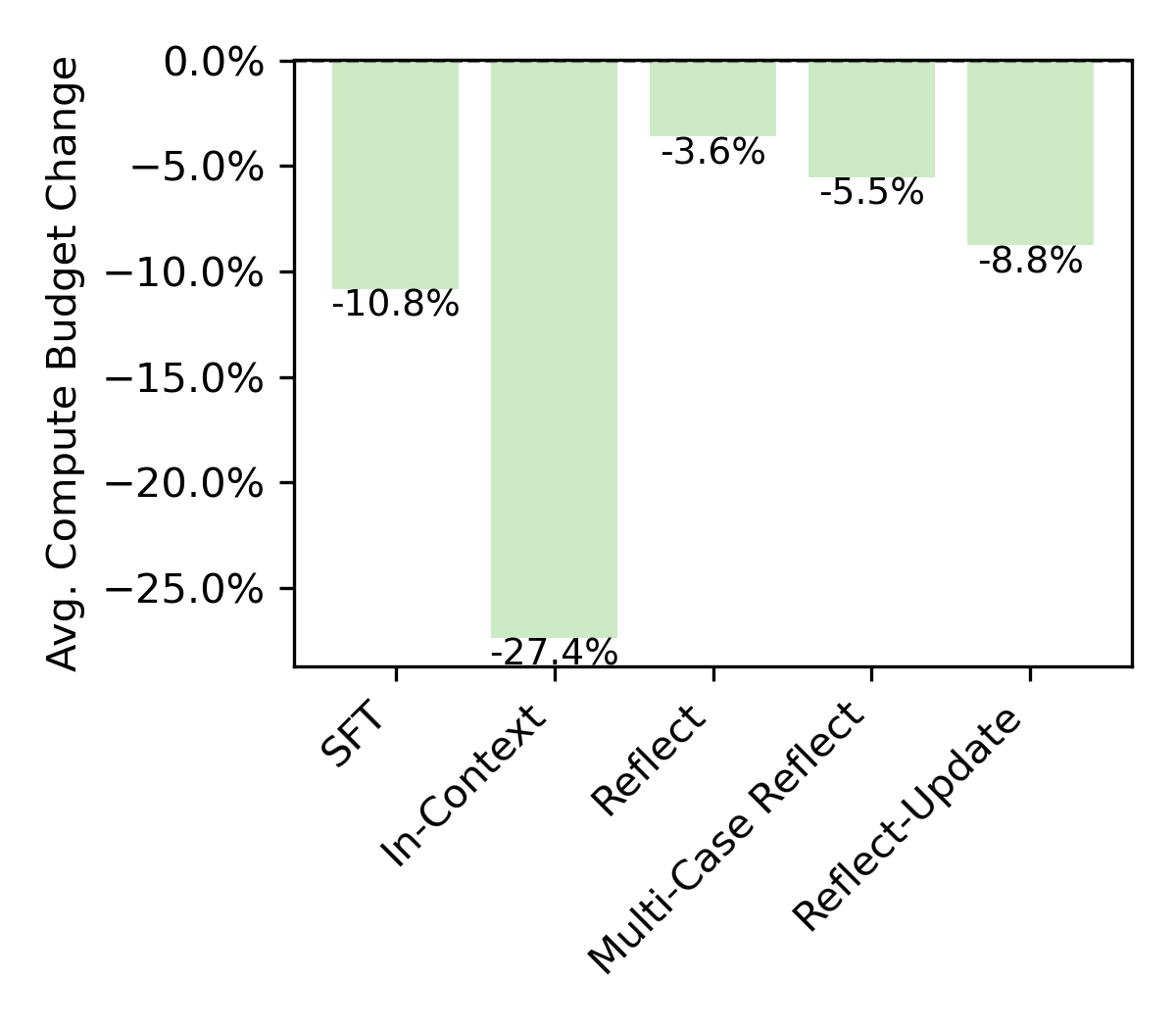
        }
        \vspace{-3mm}
        \caption{Compute budget changes grouped and averaged by memory methods.}
        \label{fig:avg-memory}
    \end{wrapfigure}
    \textit{\textbf{Finding 5. Episodic memory methods generally outperform
    semantic memory methods in LLM reasoning speedup.}}

    In Fig.~\ref{fig:avg-memory}, we show the relative compute budget values of
    each memory method, averaged across different similarity levels and test-time
    scaling methods. Generally, we find that episodic memory methods (\textit{SFT}: 10.8\%, \textit{In-Context} 27.4\%) reduce compute budgets more effectively than semantic (reflection-based) methods (3.6\%, 5.5\%, 8.8\%).
    This aligns with previous studies showing that comprehensive recall of past experience is important for benefiting LLMs in problem solving \citep{renze2024self}. Similarly, psychological research indicates that human proficiency initially relies on episodic memory, which
    allows for instance-based retrieval \citep{logan1988toward}. This suggests that episodic memory may play a critical role in initial LLM reasoning efficiency improvements.

    \vspace{1ex}
    \noindent
    \textbf{\textit{Finding 6. In-context learning yields greater improvements
    in both efficiency and accuracy than supervised fine-tuning.}}

    Comparing the \textit{in context} and \textit{SFT} rows in Fig.~\ref{fig:avg-memory}
    (also in Fig.~\ref{fig:avg_heatmap}), we find that \textit{In-Context} memory consistently gives greater reductions in compute budget and improvements in accuracy across all scaling methods
    and similarity levels. Similarly, in Fig.~\ref{fig:all_curve}~(a), the curves
    also show more consistent trends of reducing compute budget for \textit{In-Context}
    compared to \textit{SFT}, indicating that in-context memory adapts more efficiently in LLM reasoning speedup. This relates to a comparison of SFT and ICL (In-Context Learning) when the data is very few (1 to 3-shot), where SFT faces either insufficient data for adaptation \citep{luo2023empirical}
    or is prone to overfitting \citep{zhang2024scaling}, and ICL has better
    generalizability \citep{yin2024deeper}.

    \vspace{1ex}

    \begin{wrapfigure}
        {r}{0.5\textwidth}
        \centering
        \vspace{-6mm}
        \includegraphics[width=\linewidth]{
            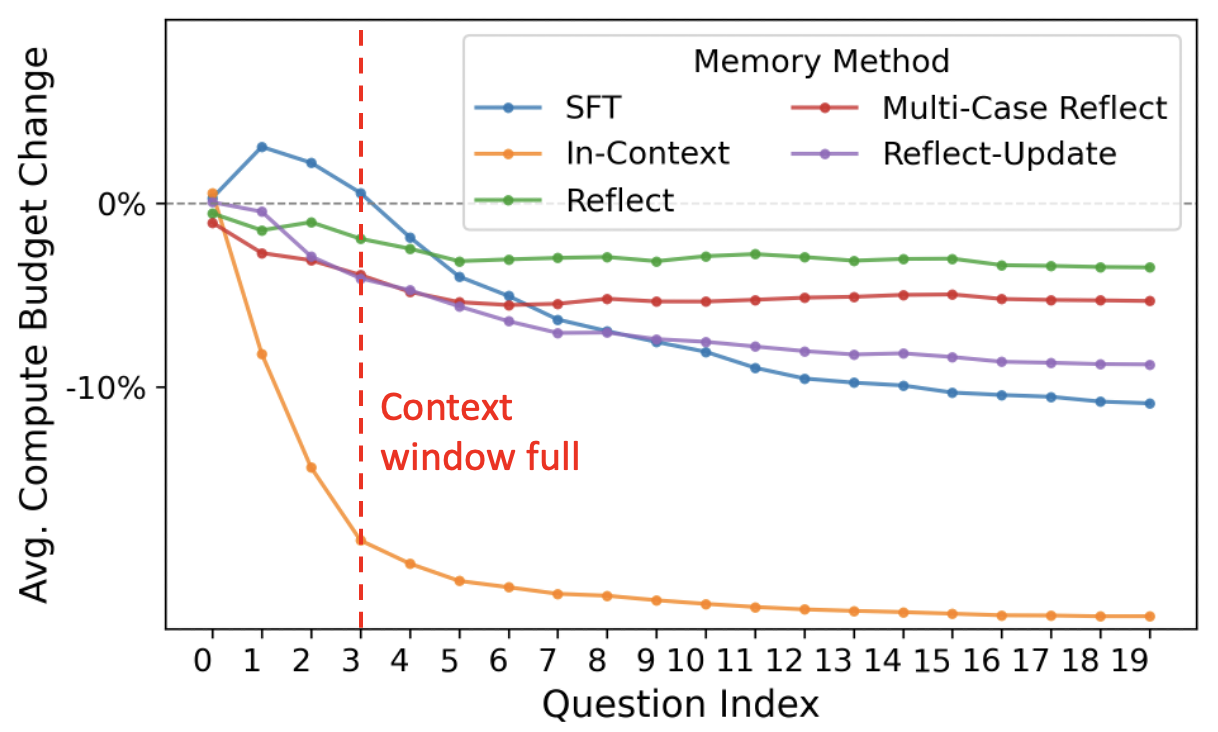
        }
        \vspace{-3mm}
        \caption{Compute budget change in each question index, grouped and averaged
        by memory methods.}
        \label{fig:curve_avg}
        \vspace{-3mm}
    \end{wrapfigure}
    \textit{\textbf{Finding 7. Text‑based memory methods plateau when the context window is full, whereas parametric memory continues to achieve reasoning speedup.}}

    In Fig.~\ref{fig:curve_avg}, we
    show the relative compute budget averaged across different similarity levels and test-time scaling
    methods. The results reveal that memory methods relying on in-context textual
    memory, including \textit{In-Context}, \textit{Reflect}, and \textit{Multi-Case
    Reflect}, exhibit diminishing efficiency gains once the context window
    becomes full (index=3). In contrast, \textit{SFT}, which saves memory into the model parameters, shows more consistent efficiency gains across the entire sequence, as it is not limited by the context window size.

    \noindent
    \textit{\textbf{Finding 8. Generalizable reflections can better help reasoning speedup.}}

    Among reflection-based methods, \textit{Reflect-Update} achieves noticeably
    stronger and more stable improvements than \textit{Reflect} and \textit{Multi-Case
    Reflect} in both efficiency and accuracy, as shown in Fig.~\ref{fig:avg-memory}
    and Fig.~\ref{fig:avg_heatmap}. Upon examining the model-generated
    reflections, we found that the reflections generated by \textit{Reflect} and
    \textit{Multi-Case Reflect} often contain specific numbers or calculations (some examples are shown in Appendix~\ref{append:res}),
    which may not be generalizable across different questions, while \textit{Reflect-Update} tends to generate
    more general reflections, reducing the risk of using unrelated calculations
    to mislead the answering of the current question. This also implies that a tradeoff between generalizability and informativeness needs to be considered in future work when trying to design better reflection-based methods for LLM reasoning speedup.

    % \vspace{-2mm}
    \subsection{The Effects of Test-Time Scaling Methods}\label{res:scaling}
    \textit{\textbf{Finding 9. The effect of different test-time scaling methods
    is correlated with memory methods.}}

    Based on our results, there is no single best test-time scaling method for enabling
    LLM reasoning speedup. As shown in Fig.~\ref{fig:avg_heatmap}, the effectiveness
    of scaling methods is closely related to the type of memory method used. When
    using episodic memory methods such as \textit{SFT} and \textit{In-Context},
    \textit{DFS} appears to be the most effective in reducing compute budgets.
    In contrast, when using reflection-based memory methods, \textit{Self-Refine}
    and \textit{Long CoT} yield better efficiency gains.
    Given the strong influence of the scaling method on the overall LLM
    performance, which is also evidenced by the fact that many state-of-the-art models
    such as OpenAI o1 \citep{jaech2024openai} and DeepSeek-R1 \citep{guo2025deepseek}
    adopt \textit{Long CoT}, it is important to recognize that the choice of scaling methods should not be the first consideration. However, for \textit{Long CoT},
    we do observe that different memory methods show similar efficiency improvements,
    with \textit{In-Context} performing slightly better. For the exploration of strategies that better combine memory methods with test-time scaling methods,
    particularly under \textit{Long CoT} settings, we leave it to future work.

    \section{Conclusion, Limitations, and Future Work}
    In this study, we raised and formally defined the question of whether large language models (LLMs) can achieve reasoning speedup through repeated exposure. To address this, we proposed {SpeedupLLM}, a unified framework for implementing and benchmarking reasoning speedup behaviors in LLMs. Through extensive experiments, we observed that reasoning speedup generally emerges across different memory mechanisms and test-time scaling methods, particularly when the questions exhibit higher similarity. Additionally, we provided several insights into the factors that influence such behaviors.
    
    This work has several limitations. First, our exploration of memory mechanisms focused on widely used methods, but it cannot cover all recent advances, e.g., compact episodic memory formats. For Long-CoT reasoning, we did not identify memory methods that significantly reduce the compute budget.
    Future work should explore more effective approaches for reasoning speedup in Long-CoT reasoning models. Exploring how to leverage in-context episodic memory while overcoming the limitations of context window length also represents a promising direction.

\bibliography{main}

\appendix
    \section{Proof of Theorems and Corollary}
    \subsection{Proof of Theorem 1}\label{proof:1}
    
        \begin{proof}

            From the assumptions we have:
            \[
                \Prob\left(\max_{1 \le j \le k}\texttt{score}(r_{j}^{(t+1)}) \ge
                \tau\right) \ge \Prob\left(\max_{1 \le j \le k}\texttt{score}(r_{j}
                ^{(t)}) \ge \tau\right).
            \]

            So the cumulative distribution of $|\mathcal{R}^{(t)}|$ satisfies:
            \[
                \Prob\bigl(|\mathcal{R}^{(t)}| \le k\bigr) = \Prob\left(\max_{1
                \le j \le k}\texttt{score}(r_{j}^{(t)}) \ge \tau\right).
            \]
            Since $\texttt{cost}(\mathcal{R})$ is non-decreasing with $|\mathcal{R}|$, the expected cost satisfies
            % \[
            %     \Prob\bigl(|\mathcal{R}^{(t+1)}| \le k\bigr) \ge \Prob\bigl(|\mathcal{R}
            %     ^{(t)}| \le k\bigr),
            % \]

            % Since $g$ is monotonically non-decreasing, by the property of stochastic
            % dominance:
            \[
                \E\bigl[\texttt{cost}(|\mathcal{R}^{(t+1)}|)\bigr] \le \E\bigl[\texttt{cost}(|\mathcal{R}
                ^{(t)}|)\bigr]. 
            \]
        \end{proof}
    % \end{proofbox}

    \subsection{Proof of Theorem 2}\label{proof:2}
    We first consider a simplified setting, where we assume the independence of each $r\in \Rcal_{s} ^{(t)}$. Note this is valid without simplification for independent sampling test-time scaling methods, e.g. Best-of-N.
    \begin{proof}\label{proof:2:actual}
    Let
    \(p_t := \Prob\bigl(\score(r_1^{(t)})\ge\tau\bigr)\).
    Because memory grows monotonically with $t$, we have $p_{t+1}\ge p_t$. Under the independence assumption,
    \[
      \Prob\left(\max_{1\le j\le k}\score(r_j^{(t)})\ge\tau\right)
      = 1 - (1-p_t)^{k},
    \]
    As this function is monotonically increasing in $p_t$, we conclude
    \[
      1-(1-p_{t+1})^{k}\;\;\ge\;\;1-(1-p_t)^{k},
    \]
    which proves the claim
    $$\Prob\left(\max_{1 \le j \le k}\texttt{score}(r_{j}^{(t+1)})\ge\tau\right)>\Prob\left(\max_{1 \le j \le k}\texttt{score}(r_{j}^{(t)})\ge\tau\right).$$

    \end{proof}

    % Next, we generalize the proof by removing the independence assumption and considering the case where candidates $r_j^{(t)} \in \Rcal_s^{(t)}$ are dependent, but satisfy a natural topological generation structure.

    % \begin{proof}
    % Suppose that for each $t$, the candidates $\Rcal_s^{(t)} = \{r_1^{(t)}, r_2^{(t)}, \dots, r_k^{(t)}\}$ are generated in a topological order such that for any $i < j$, the generation of $r_j^{(t)}$ may conditionally depend on $r_i^{(t)}$, but not vice versa. This condition captures many real-world scaling methods such as self-refining generation and DFS search, where later steps may refine or expand earlier ones. Given the question similarity, we can assume that $\Rcal_s^{(t)}$ and $\Rcal_s^{(t+1)}$ follows the same dependency pattern, so we have
    % \[
    %     \E[\score(r_j^{(t+1)})] \ge \E[\score(r_j^{(t)})].
    % \]
    % Then, for each $j$, we have:
    % \[
    %     \Prob(\score(r_j^{(t+1)}) \ge \tau) \ge \Prob(\score(r_j^{(t)}) \ge \tau).
    % \]

    % \end{proof}

    Next, we generalize the proof by removing the independence assumption and considering the case where candidates $r_j^{(t)} \in \Rcal_s^{(t)}$ are dependent, but satisfy a natural topological generation structure. We further assume:

\begin{enumerate}
    \item[\textbf{(A1)}]   
    There exists a directed acyclic graph (DAG)
    $\mathcal{G}=(V,E)$ on the index set $V=\{1,\dots,k\}$ that is
    \emph{identical} for $t$ and $t+1$.  
    We write
    $
      r_{i}^{(t)} \,\leadsto\, r_{j}^{(t)}
    $
    if $(i,j)\in E$, i.e.\ candidate $r_{j}^{(t)}$ \emph{relies on}
    $r_{i}^{(t)}$.  
    For each $j$ let $\mathcal{P}(j)=\{\,i:(i,j)\in E\}$ be its (possibly
    empty) parent set.

    % \item[\textbf{(A2)}] \textbf{Root dominance.}  
    % If $\mathcal{P}(j)=\varnothing$ (\emph{roots of $\mathcal{G}$}),
    % then additional memory never harms their quality:
    % \[
    %     \score\bigl(r_{j}^{(t+1)}\bigr)\;\gst\;\score\bigl(r_{j}^{(t)}\bigr).
    % \]

    \item[\textbf{(A2)}] 
    For any non-root node $j$ the conditional distribution of
    $\score(r_{j}^{(t)})$ is monotone in its parents’ scores:
    if $\mathbf{s}\succcurlyeq\mathbf{s}'$ coordinate-wise then
    \[
        \score\bigl(r_{j}^{(t)}\mid\score(\mathcal{P}(j))=\mathbf{s}\bigr)
        \;\gst\;
        \score\bigl(r_{j}^{(t)}\mid\score(\mathcal{P}(j))=\mathbf{s}'\bigr).
    \]
    Consequently, whenever every parent’s score distribution
    stochastically improves, so does the child’s.
\end{enumerate}

\begin{proof}
We first convert the \emph{mean}-improvement assumption  
\[
   \E\bigl[\score(r_j^{(t+1)})\bigr]
   \;\ge\;
   \E\bigl[\score(r_j^{(t)})\bigr]  \tag{B0}
\]
into a tail-probability guarantee for every \emph{root} node
\(j\) (nodes with \(\mathcal{P}(j)=\varnothing\)).
When scores are normalized to the unit interval,
the following elementary bound holds.

\begin{lemma}\label{lem:mean2tail}
Let \(S\in[0,1]\) and \(\tau\in(0,1)\). Then
\[
   \Prob(S\ge\tau)\;\;\ge\;\;\dfrac{\E[S]-\tau}{1-\tau}.
\]
\end{lemma}

\begin{proof}[Proof of Lemma~\ref{lem:mean2tail}]
Write
\(
  \E[S]=
  \E[S\mid S\ge\tau]\Prob(S\ge\tau)+
  \E[S\mid S<\tau]\bigl(1-\Prob(S\ge\tau)\bigr).
\)
Because \(S\le1\) and \(S<\tau\) on the second event,
\(\E[S\mid S<\tau]\le\tau\). Hence  
\(
  \E[S]\le
  \Prob(S\ge\tau)+
  \tau\bigl(1-\Prob(S\ge\tau)\bigr)
  =\tau+(1-\tau)\Prob(S\ge\tau),
\).
\end{proof}

\smallskip
Applying Lemma~\ref{lem:mean2tail} to~(B0) gives, for every root \(j\),
\[
   \Prob\bigl(\score(r_j^{(t+1)})\ge\tau\bigr)
   \;\ge\;
   \Prob\bigl(\score(r_j^{(t)})\ge\tau\bigr) \quad \forall j\;\text{s.t.}\;\mathcal{P}(j)=\varnothing,
\]
i.e.\ \(\score(r_j^{(t+1)})\gst\score(r_j^{(t)})\).  

Then, we traverse the fixed dependency DAG \(\mathcal{G}\) in
a topological order.  
Assume for every parent \(i\in\mathcal{P}(j)\) of the current node \(j\)
that
\(
  \score(r_i^{(t+1)}) \gst \score(r_i^{(t)}).
\)
By construction the joint parent-score vector at \(t+1\) stochastically dominates the
one at \(t\).  
Under the \emph{monotone-reliance} assumption (B3),
this dominance propagates to the child:
\[
   \score(r_j^{(t+1)}) \gst \score(r_j^{(t)}).
\]
Thus the induction hypothesis extends to every successor node.
Proceeding through the whole order yields
\[
   \score(r_j^{(t+1)}) \gst \score(r_j^{(t)})
   \quad\forall j.
\]

Thus, the mapping
\(g(\mathbf{x})=\max_{1\le j\le k}x_j\)
is monotone, i.e.
\[
   \max_{j}\score(r_j^{(t+1)}) \gst \max_{j}\score(r_j^{(t)}).
\]
So for every threshold \(\tau\), we have
\[
   \Prob\Bigl(
      \max_{1\le j\le k}\score(r_j^{(t+1)})\ge\tau
   \Bigr)
   \;\ge\;
   \Prob\Bigl(
      \max_{1\le j\le k}\score(r_j^{(t)})\ge\tau
   \Bigr).
\]
\end{proof}

\section{Test-time Scaling Methods}\label{append:scaling_method}
 In this appendix, we elaborate how several widely-used test-time scaling methods can be expressed within the unified framework introduced in Section~X.
Let $f_s(q^{(t)})$ denote the set of candidate answers $\mathcal{R}_s^{(t)}$ generated by applying test-time scaling method $s \in \mathcal{S}$ to question $q^{(t)}$. Each method defines a minimal unit of computation, a corresponding cost function $\texttt{cost}(\cdot)$, and a score function $\texttt{score}(\cdot)$ to evaluate answer quality.
\vspace{1ex}

\noindent\textbf{Best-of-N.}
This method generates $N$ complete candidate answers. Each $r_{k;s}^{(t)} \in \mathcal{R}{s}^{(t)}$ is an independently sampled answer.
\begin{itemize}
\item \textit{Unit of computation:} full answer.
\item \textit{Cost:} number of generated answers, i.e., $|\mathcal{R}s^{(t)}|$.
\item \textit{Score:} LLM Judge or Process Reward Model assigns $\texttt{score}(r)$ for each candidate $r$.
\item \textit{Final answer:} $\argmax{r \in \mathcal{R}s^{(t)}} \texttt{score}(r)$.
\end{itemize}
This method directly aligns with our unified formulation, with $f_s$ sequentially or in batches generating $r{1}, r{2}, \dots, r_{N}$ until an answer satisfying the threshold is found.
\vspace{1ex}

\noindent\textbf{Tree-of-Thoughts.}
This method performs structured reasoning by expanding nodes in a search tree.
\begin{itemize}
\item \textit{Unit of computation:} tree node.
\item \textit{Cost:} number of generated and scored nodes.
\item \textit{Score:} each node is scored via LLM Judge or reward model.
\item \textit{Final answer:}  the highest-scoring reasoning path.
\end{itemize}
Our formulation covers this by letting $f_s$ denote the expansion of search tree nodes, with $\mathcal{R}_s^{(t)}$ representing partial reasoning paths. Both DFS and BFS are special cases, differing in node expansion order.
\vspace{1ex}

\noindent\textbf{Self-Refine.}
This method iteratively generates improved answers based on previous attempts.
\begin{itemize}
\item \textit{Unit of computation:} full refined answer.
\item \textit{Cost:} number of refinement steps (full answers).
\item \textit{Score:} internally assessed (e.g., by the model itself) or by an external judge.
\item \textit{Final answer:} the most refined answer exceeding the threshold.
\end{itemize}
Self-Refine is inherently adaptive: $f_s$ produces a sequence of revised answers until $\texttt{score}(r) \ge \tau$. No explicit compute optimization is required.
\vspace{1ex}

\noindent\textbf{Long CoT (Chain-of-Thought).}
This method enables free-form, unbounded reasoning over long text segments.
\begin{itemize}
\item \textit{Unit of computation:} token.
\item \textit{Cost:} number of generated tokens.
\item \textit{Score:} assessed implicitly by the model (e.g., via stopping probabilities).
\item \textit{Final answer:} the whole generated sequence.
\end{itemize}
This method integrates self-refinement and tree-search behavior into token-level generation. $\mathcal{R}_s^{(t)}$ here can be seen as the full text trace, and $\texttt{score}(\cdot)$ may correspond to internal confidence or the model’s own stop criterion (e.g., omitting “wait” tokens).

\section{Data Examples}\label{append:data}
\textbf{Question Backbone}: 
\begin{lstlisting}
Let $\mathbf{a} = \begin{pmatrix} 1 \\ 1 \\ 0 \end{pmatrix}$ and $\mathbf{b} = \begin{pmatrix} 2 \\ 0 \\ -1 \end{pmatrix}.$  Find the vector $\mathbf{v}$ that satisfies $\mathbf{v} \times \mathbf{a} = \mathbf{b} \times \mathbf{a}$ and $\mathbf{v} \times \mathbf{b} = \mathbf{a} \times \mathbf{b}.$
\end{lstlisting}

\textbf{S1}: (Exactly the same with the question backbone.)

\textbf{S2} (Same numbers, different wording):
\begin{lstlisting}
Given the vectors $\\mathbf{a} = \\begin{pmatrix} 1 \\\\ 1 \\\\ 0 \\end{pmatrix}$ and $\\mathbf{b} = \\begin{pmatrix} 2 \\\\ 0 \\\\ -1 \\end{pmatrix}$, determine the vector $\\mathbf{v}$ that meets the conditions $\\mathbf{v} \\times \\mathbf{a} = \\mathbf{b} \\times \\mathbf{a}$ and $\\mathbf{v} \\times \\mathbf{b} = \\mathbf{a} \\times \\mathbf{b}$.
\end{lstlisting}

\textbf{S3} (Same structure, different numbers):

\begin{lstlisting}
Let $\\mathbf{a} = \\begin{pmatrix} 0 \\\\ -3 \\\\ -2 \\end{pmatrix}$ and $\\mathbf{b} = \\begin{pmatrix} -1 \\\\ -3 \\\\ 0 \\end{pmatrix}.$  Find the vector $\\mathbf{v}$ that satisfies $\\mathbf{v} \\times \\mathbf{a} = \\mathbf{b} \\times \\mathbf{a}$ and $\\mathbf{v} \\times \\mathbf{b} = \\mathbf{a} \\times \\mathbf{b}.
\end{lstlisting}

\textbf{S4} (Same underlying knowledge, different structure and numbers):
\begin{lstlisting}
Find the scalar $k$ such that the points $(1, k)$, $(k, 2)$, and $(3, 4)$ are collinear.
\end{lstlisting}

\section{Implementation Details}\label{append:implement}
We develop a unified test-time reasoning framework in Python, using PyTorch and Hugging Face's Transformers. Our experiments are conducted on machines with NVIDIA A100, H100, H200, L40S and L4 GPUs. 

\textbf{Base Model.} Experiments on Best-of-N, DFS and Self-Refine are conducted on the Llama-3.1-8B model, and experiments on Long CoT are conducted on the DeepSeek-R1-Distill-Qwen-7B model. For value estimation, we optionally use a separate lightweight LLM, \texttt{gpt-4o-mini}. For all generation, we use temperature of 0.7 and top\_p of 0.9.

\paragraph{Test-Time Scaling Methods.} Below, we detail their core mechanisms, configurable parameters, and termination criteria:

\begin{itemize}
    \item \textbf{Best-of-N}: Generates $N$ candidates (set via \texttt{--n\_generate\_sample}, default = 5). Each is scored using a value model or PRM (\texttt{--method\_evaluate}), and a candidate is selected based on value and correctness.  
    \textit{Termination:} stops early if a candidate meets both the score threshold (\texttt{value\_thresh} = 0.9) and correctness; otherwise, completes all $N$ evaluations.

    \item \textbf{Self-Refine}: Begins with a generated candidate and iteratively improves it using feedback and refinement prompts.  
    \texttt{num\_iteration} (default = 15) sets the max refinement steps.  
    \textit{Termination:} stops early if the feedback indicates "No error"; otherwise continues until the max iteration count is reached.

    \item \textbf{Long CoT}: A single long-form reasoning trace is generated, typically prefixed with a \texttt{<think>} tag. The model self-determines when to stop generating, often marked by the token \texttt{</think>}.  
    \textit{Termination:} when the model stops generating or produces an explicit termination tag. We set a max token number of each answer as 3500 tokens.

    \item \textbf{Depth-First Search (DFS)}: A search tree is constructed over reasoning steps, where nodes represent partial reasoning segments. Child nodes are generated up to \texttt{max\_depth} (default = 15), and evaluation is guided by value thresholds (\texttt{value\_thresh}) and pruning ratio (0.4).  
    \textit{Termination:} when either a  termination node (the content has "End of Answer" with perfect value (1.0) is found, or the search exceeds \texttt{max\_node} (default = 50) expansions, or all candidates are exhausted.
\end{itemize}

\textbf{Memory Methods.} We implement the following memory mechanisms.

\begin{itemize}
    \item \textbf{No Memory}: The baseline configuration with no carry-over between rounds. 

    \item \textbf{Supervised Fine-Tuning (SFT)}: The model is updated after each question using a single-step gradient descent. For reasoning models (e.g., \texttt{DeepSeek-R1}), we generate a compressed representation of the reasoning trace using a summarization prompt. The fine-tuning is performed using the following hyperparameters: \texttt{learning\_rate} = 5e-4, Single-step update with LoRA enabled.
    The SFT method is only applied if a correct and high-quality answer is found.

    \item \textbf{In-Context}: A memory buffer maintains the last $n$ successful examples. For each new prompt, we prepend these examples as demonstrations. We set the maximum number of in-context examples to 3.

    \item \textbf{Reflection}: After a correct response, a language model is prompted to reflect on the reasoning process. The resulting reflection is stored and used in future prompts under the "Consider:" section. Only one reflection is stored per example.

    \item \textbf{Multi-Case Reflection}: Instead of prompting reflection on a single instance, this method generates a joint reflection across multiple past successful cases. All stored examples are included in a joint input to generate an abstracted reflection. We limit the context by a maximum of 3 examples.

    \item \textbf{Reflect-Update}: This method iteratively refines a single running reflection. After each correct answer, the previous reflection and the new case are used to generate an updated reflection. The updated reflection replaces the old one, maintaining a compact and evolving summary of reasoning strategies. 
\end{itemize}

\paragraph{Memory Update Policy.}
Memory is updated only if a qualifying answer is found during the current question’s rounds (i.e., correct and exceeds score threshold). In SFT-based methods, this data is also used to update the parametric weights of the inference or reward model.

\paragraph{Evaluation Configuration.}
We run the evaluation over 10 question sets, with 4 repetitions for each question, to report the mean value.

\section{Additional Results}\label{append:res}
\subsection{Examples of Different Reflection Methods}

Here we show the reflection generated from the same question backbone.

\textbf{Reflect (individually): }(\textbf{This is an unrelated reflection derived from a less-similar question in the same set with Similarity S4.})

\texttt{
- Identify the principal amount (\$1,000), the interest rate (5\%), and the time period (3 years).\\
- Use the formula for simple interest: Interest = Principal × Rate × Time.\\
- Substitute the known values into the formula: Interest = \$1,000 × 0.05 × 3.\\
- Calculate the interest: Interest = \$150.\\
End of answer.
}

\textbf{Multi-Case Reflect}:

\texttt{
- Define the initial positions of the train and car as (0, 0).\\
- Let the speed of the train be \(v_t\) and the speed of the car be \(v_c\).\\
- After time \(t\), the position of the train will be \((0, v_t \cdot t)\) and the position of the car will be \((v_c \cdot t, 0)\).\\
- Use the distance formula \(d = \sqrt{(x_2 - x_1)^2 + (y_2 - y_1)^2}\) to express the distance \(D\) between the train and car as:
  \[
  D = \sqrt{(v_c \cdot t - 0)^2 + (0 - v_t \cdot t)^2} = \sqrt{(v_c \cdot t)^2 + (v_t \cdot t)^2}
  \]\\
- Simplify the equation to:
  \[
  D = t \cdot \sqrt{v_c^2 + v_t^2}
  \]\\
- Substitute the known values of \(t\), \(v_c\), and \(v_t\) to find the distance \(D\) after the specified time.\\
End of answer.
}

\textbf{Reflect-Update: }

\begin{lstlisting}
- Reaffirmed the use of the centripetal acceleration formula \( a_c = \frac{v^2}{r} \) as a fundamental principle in circular motion.\\
- Clarified the substitution of values, ensuring each step is transparent and logically follows from the previous one, which aids in comprehension.\\
- Highlighted the importance of checking calculations at each stage, particularly in deriving \( a_c \) and converting to terms of \( g \).\\
- Emphasized the need for consistent units throughout the calculation to avoid errors, reinforcing unit conversion practices.\\
- Improved the rounding approach to maintain precision in the final result, ensuring it aligns with significant figures based on input values.\\
- Incorporated a systematic verification step to cross-check the final answer against physical principles, ensuring it is reasonable and accurate.\\
- Continued to prioritize expressing results in terms of \( g \), fostering a deeper understanding of the practical implications of centripetal acceleration in real-world scenarios.\\
End of answer.

\end{lstlisting}

\end{document}